\documentclass[12pt]{article}
\usepackage[utf8]{inputenc}

\usepackage{allrunes}
\usepackage{amsmath}
\usepackage{geometry}
\usepackage{amsthm}
\usepackage{xcolor}
\usepackage{url}
\usepackage[T1]{fontenc}

\usepackage{multicol}
\usepackage{dsfont}

\newcommand{\xm}{\textbf{X}}

\newcommand{\um}{\textbf{U}}
\newcommand{\wm}{\textbf{W}}

\newcommand{\km}{\textbf{K}}
\newcommand{\vm}{\textbf{V}}

\newcommand{\sv}{\textbf{s}}
\newcommand{\xv}{\textbf{x}}
\newcommand{\yv}{\textbf{y}}
\newcommand{\vv}{\textbf{v}}
\newcommand{\hv}{\textbf{h}}

\newcommand{\svec}{\textbf{s}}
\newcommand{\nv}{\textbf{n}}
\newcommand{\av}{\textbf{a}}

\newcommand{\qv}{\textbf{q}}

\title{\textbf{A Survey of Neural Networks and Formal Languages}
\footnote{This work was partially supported by DARPA Safedocs Program award HR001119C0075 for which SRI is the prime contractor and Dartmouth is a subcontractor.}}

\author{Joshua Ackerman\thanks{\noindent Department of Computer Science, Dartmouth College, Hanover NH 03755. joshua.m.ackerman.gr{@}dartmouth.edu} \and 
George Cybenko {\footnote{ Thayer School of Engineering, Dartmouth College, Hanover NH 03755. gvc{@}dartmouth.edu }}}

\date{\today}

\usepackage{graphicx}

\newtheorem{definition}{Definition}
\newtheorem{theorem}{Theorem}

\begin{document}

\maketitle

\begin{abstract}
    This report is a survey of the relationships between
    various state-of-the-art neural network architectures and formal languages as, for example, structured by the Chomsky Language Hierarchy.  Of particular interest are the abilities of a neural architecture to represent, recognize and generate words from a specific language by learning from samples of the language.
\end{abstract}

\section{Introduction}
Understanding how well different neural network architectures decide membership in classes of formal languages is a fundamental problem spanning machine learning and language processing. In principle, it is known that even the simplest variants of Recurrent Neural Networks are capable of emulating a Turing Machine in real-time, and consequently are Turing Complete \cite{sieg}. However, this result and similar ones, rely on unrealistic assumptions such as unbounded computation time and infinite precision representation of real-valued states. 

For these reasons, our understanding of the relationship between different network architectures and the Chomsky hierarchy \cite{jager2012formal} is not complete, and as a consequence there is growing interest in understanding how different networks operate under these more realistic constraints. In particular, we consider computation time linear in the input size and bounded precision numbers as constraints. Adopting the terminology of \cite{weiss}, we will describe such networks as input-bounded neural networks with finite-precision states (IBFP-NNs). 

Research in this general area dates back to the early 1990s, with works such as \cite{steijvers1996recurrent} \cite{tonkes} \cite{holldobler} \cite{rodriguez} empirically studying the ability of different networks to learn variations of context-free counter languages \cite{fischer1968counter}. Notably, it was around this time when researchers started exploring the idea of augmenting networks with memory constructs such as in \cite{das}, which introduced the Recurrent Neural Network Pushdown Automaton (NNPDA) -- an RNN augmented with an external stack. Very recently, the body of research on such Memory Augmented Neural Networks has grown considerably, with the introduction of fully differentiable memory models such as Neural Stacks \cite{gref}, Neural Queues \cite{gref}, and even Neural Turing Machines \cite{graves}. Many of these papers present empirical results, so naturally, it is not guaranteed their findings hold in general. 

However, in the last year, a number of papers have appeared that make very direct comparisons of modern network architectures with respect to formal language processing. Papers such as \cite{mer} and \cite{weiss} explore the theoretical power of IBFP variants of many such state-of-the-art networks, while papers like \cite{suz} explore how networks can be augmented to better perform formal language focused tasks.

For the SafeDocs project, we are primarily focused on two key tasks: (i) generating novel samples from unknown grammars given a small sample set of positive examples; (ii) empirically deciding whether new samples are in a language thus learned. Although the aforementioned papers primarily discuss results in the spirit of task (ii), it is worth noting that generative modeling (i), is a comparably harder task than that of discriminative modeling (ii), so many of these results are still applicable (albeit less precisely) with respect to quantifying the limits of networks towards generative tasks. 

We believe our empirical exploration as part of task (ii) will motivate richer theoretical results surrounding the generative capacity of certain models in the setting of Generative Adversarial Neural Networks \cite{ian}. For the moment, however, the most relevant results to our work are those of papers like \cite{weiss} \cite{mer} \cite{suz} which study the power or design of state-of-the-art networks in the context of language recognition.  In this project oriented survey, we will walk through some of the most pertinent results from these papers, surveying theoretical and empirical results which quantify the ability of Convolutional Neural Networks (CNNs),  Recurrent Neural Networks (RNNs), Transformer Networks, along with variations on Memory Augmented Neural Networks (MANNs) to learn languages in the Chomsky Hierarchy under the practical assumptions.

This report is structured as follows:  Section \ref{sec-def} reviews several concepts related to network performance and language theory. Section \ref{sec-results} reviews known results regarding several network architecture types with respect to language representation and recognition problems.  Section \ref{sec-summary} is a succinct summary of the results we have reviewed in this report.

\section{Definitions} \label{sec-def}
In this section we recall the core terminology introduced in \cite{mer} which will be necessary to properly contextualize the concept of asymptotic analysis of neural networks. To start, given an alphabet $\Sigma$ of size $\ell$, we encode an $n$-length sequence as an $n \times \ell$ matrix $\xm$ where the $i$th row of $\xm$, denoted $\xv_i$, is one-hot encoding of the $i$th sequence character. With this basic notion in mind, we can define the fundamental concept of a \textit{neural sequence acceptor} \cite{mer}.

\begin{definition}[Neural Sequence Acceptor] A neural sequence acceptor $\hat{\mathds{1}}$ is a family of functions, parameterized by $\Theta$, of the form
\[
    (\hat{\mathds{1}}^\Theta : \xm \mapsto p)_{\xm \in {\{0,1\}^{n \times \ell}}}
\]
where $p \in [0,1]$.
\end{definition}
\noindent Intuitively, neural sequence acceptors are neural networks with parameters $\Theta$ which take sequences as input, and returns the probability that the input is part of some language. 
\\ \\ All presented results from  \cite{mer}, are quantified by the notion of \textit{asymptotic acceptance}. Very simply, in this setting we allow the magnitude of the parameters $\Theta$ of a neural node to get arbitrarily large. Of course, this alone is not intrinsically a practical assumption however, unintuitively, it actually leads to a more pessimistic view on the capacity of the computational power of IBFP-NNs. In practice, neural networks have been observed to have the ability to learn non-asymptotic strategies in certain problems \cite{mer}. However, small noisy perturbations of the activation functions of those networks during training lead to failure on the tested problems \cite{mer}, which suggests that on ``noisier'' datasets (i.e., PDFs), asymptotic results are in fact more likely to be realized, and therefore are not an unreasonably pessimistic or irrational assumption in practice.

\begin{definition}[Asymptotic Acceptance] Let $L$ be a language with indicator function $\mathds{1}_L$, then a neural sequence acceptor $\hat{\mathds{1}}^\Theta$ is said to asymptotically accept $L$ if
    \[
        \lim_{N \rightarrow \infty} \hat{\mathds{1}}^{N\Theta} = \mathds{1}_L.
    \]
\end{definition}
\noindent Beyond the notion of acceptance, another complexity metric for a IBFP-NN arising in this asymptotic context is that of \textit{general state complexity}. General state complexity captures the full amount of memory a network can employ at each stage of a computation. Therefore, understanding the general state complexity of a network yields insight to its expressive capacity. A networks general state complexity is measured across its \textit{hidden states}, or the intermediate representations a network employs to arrive at a final answer.
\begin{definition}[Hidden State]  Let $\vv \in \mathds{R}^k$. The $k$-length hidden state, with respect to parameters $\theta$ is a family of functions,
    \[
        (\hv^\theta_t : \xm \mapsto \vv_t)_{|\xm| = n}
    \]
defined for every $t \in [n]$.
\end{definition}
\noindent For a given hidden state with fixed parameters, the configuration set is the set of values the hidden state takes, varied over possible different input sentences, or more simply, the number of configurations a hiddens state can exist in.
\begin{definition}[Configuration Set] For all $n$, the configuration set of a hidden state $\hv_n$ with parameters $\theta$, is defined as
    \[
        M(\hv_n^\theta) = \left\{ \lim_{N \rightarrow \infty} \hv^{N\theta}_n (\xm) \mid \xv_i \in \texttt{onehot}(\Sigma) \right\}
    \]
\end{definition}
\noindent Finally, the quantity of authentic interest for a given network is the general state complexity, which considers the largest number of values a network hidden state can take across all possible model parameters.
\begin{definition}[General State Complexity] The general state complexity of a hidden state $\hv_n$ with parameters $\theta$ is defined as the maximum fixed state complexity, or,
    \[
        \textarc{m}(\hv_n) = \max_\theta \left|M(\hv_n^\theta)\right|.
    \]
\end{definition}

\subsection{Language Theoretic}
We use $\mathcal{L}(M)$ to denote the language accepted by some machine $M$, and language classes such as the set of all regular languages, \textbf{REGULAR}, will be set in capitalized, bold text. We will assume knowledge of basic formal language theory, however we will define some possibly less common classes of languages. The first such language is a \textit{strictly $k$-local language}. Fittingly, these are languages with very structured, local behavior of size $k$.
\begin{definition}[Strictly $k$-local language] Let $\Sigma$ be an alphabet, which without loss of generality does not contain the character \#. A \textit{strictly $k$-local language} is a set of constraints $S$ of the form,
    \[
        S = \left\{s \mid s \in \big(\Sigma \cup \{\#\}\big)^k \right\}.
    \]
\end{definition}
\noindent Next, we introduce \textit{Dyck Languages}, which informally consist of correctly nested sequences of parentheses. 
\begin{definition}[Dyck Languages] Given a bipartite set of characters $(P, \overline{P})$, the Dyck language, $\mathcal{D}_P$, is defined by the set,
    \[
       \mathcal{D}_P = \left\{x \in (P \cup \overline{P})^* \mid x \text{ is a well balanced set of parenthesis} \right \}.
    \]
\end{definition}
\noindent In contexts where we are only concerned with the number of parenthesis, we will write $\mathcal{D}_n$ as short hand for $\mathcal{D}_{[2n]}$. 
\noindent Although seemingly very simple, the Chomsky–Sch\"{u}tzenberger representation theorem suggests these languages are in some sense related to context free languages. This particular relationship, relies on the idea of a \textit{homomorphism}, which for two alphabets $\Sigma, \Delta$, is a map $h:\Sigma^* \mapsto \Delta^*$ such that $h(\epsilon) = \epsilon$, and for any $x,y \in \Sigma$, $h(xy) = h(x)h(y)$. 

\begin{theorem}[Chomsky–Sch\"{u}tzenberger Representation Theorem] A language $L$ over $\Sigma$ is context-free if and only if there is a bipartite set of characters $(P, \overline{P})$, a regular language $R$ over $(P, \overline{P})$, and a homomorphism $h :(P, \overline{P})^* \mapsto \Sigma^*$ such that $ L = h(\mathcal{D}_{P} \cap R)$.
\end{theorem}

\noindent Finally, we mention the informal definitions for \textit{simplified $k$-counter machines} found in  \cite{weiss}, and the more general notion of a \textit{counter machine} from \cite{mer}. First, Weis et al. \cite{weiss} loosely define a \textit{counter device} as something holding a value which can be incremented by a fixed amount, decremented by a fixed amount, or compared to zero (\texttt{COMP0}). They then define a \textit{simplified $k$-counter machine} (SKCM) as ``a finite-state automaton extended with $k$-counters, where at each step of the computation each counter can be incremented, decremented or ignored in an input-dependent way, and state-transitions and accept/reject decisions can inspect the counters' states using \texttt{COMP0}.'' Accordingly, counter machines are simply finite automata which have access to a finite number of counting devices \cite{mer}.

\section{Results} \label{sec-results}
\subsection{Convolutional Neural Networks}
Networks such as Recurrent Neural Networks, tend to be more colloquially associated with sequential data, however Convolutional Neural Networks (CNNs) do have compelling uses in sequential tasks \cite{yin} due to  strengths such as their ability to learn positionally invariant features. We are not aware of any results concerning the capacity of deep convolutional networks in the context of formal language tasks, however \cite{mer} studied a simple, convolutional networks with max pooling in the IBFP setting. The model of a CNN which they study is given by the following equations
    \begin{align}
        &\hv_t = \tanh \left( \textbf{W}^h \left(\xv_{t-k} || \cdots || \xv_{t+k} \right) + \textbf{b}^h \right) \\
        &\hv_\text{pool} = \text{maxpool}(\textbf{H}) \\
        &a = \sigma(\textbf{W}^a \hv_\text{pool} + \textbf{b}^a).
    \end{align}
\noindent Based on this model, \cite{mer} proved the following result.
\begin{theorem} Let \textbf{REGULAR} be the class of all regular languages, and  \textbf{STRICTLY-LOCAL} be the class of all languages acceptable by a strictly local grammar. The following inclusions hold asymptotically,,
   \[
    \textbf{STRICTLY-LOCAL} \subseteq \mathcal{L}(\textsf{IBFP-1-CNN}) \subset \textbf{REGULAR}.
   \]
\end{theorem}
\begin{proof} In the interest of space, we show only $\mathcal{L}(\textsf{IBFP-1-CNN}) \subset \textbf{REGULAR}$, for the proof of the other containment see \cite{mer}. Consider the language $L$ given by the regular expression $a^* b a^*$, and towards a contradiction assume that there is a \textsf{IBFP-1-CNN} that accepts $L$. Consider a string $\xv \in L$, with a $\xv_i = b$. Pick a $j$ such that $|i - j| > 2k+1$, and change $\xv_j = b$. No column in the network will get $\xv_i$ and $\xv_j$ as input, so perthe maxpooling step, the network will accept.
\end{proof}
\noindent Even though this result is not true for deep convolutional network, it seems unlikely that nesting layers would yield too much additional formal capacity in the IBFP setting. So, it is possible that, even more generally CNNs could have some intrinsic limitations.

\subsection{Recurrent Neural Networks}
In contrast to Convolutional Neural Networks, Recurrent Neural Networks (RNNs) were specifically designed with sequential data in mind, and as such have been one of the most ubiquitous paradigms in fields centered around highly sequential like natural language processing. Roughly speaking, for each token $\xv_t$ in the input RNNs compute a state $\hv_t$ using the preceding state and the new token, i.e., $\hv_t = f(\xv_t, \hv_{t-1})$ for some nonlinear function $f$. One of the simplest realizations of an RNN is the Simple Recurrent Neural Network (SRNN) \cite{elman}, essentially just applies an affine transformation to $\xv_t$, $\hv_{t-1}$, before a non-linear function. The full SRNN network is defined by these update rules, 
    \begin{align}
        &\hv_t = f_h(\wm \xv_t + \um \hv_{t-1} + b_h) \\
        &\yv_t = f_y(\wm^y \hv_t + b_y).
    \end{align}
Although powerful, these networks have issues such as being highly susceptible to vanishing gradients \cite{lstm}, which motivated the development of more complex networks such as the Long Short Term Memory (LSTM) \cite{lstm}, or the Gated Recurrent Unit (GRU) \cite{gru}. LSTMs improve SRNNs by using more complicated rules to decide how information passes between states. Intuitively, LSTMs add ``input'', ``ouput'', and ``forget'' gates, $\textbf{i}_t, \textbf{o}_t,  \textbf{f}_t$, which control how values move between memory cells -- the fundamental units of the LSTM. Memory cells have state $\textbf{c}_t$, and input activation $\tilde{\textbf{c}}_t$. The input gate and output gates control how values come into and how the value affects the activation of a cell, meanwhile, the forget gate controls the extent to which a value remains in a cell. Compactly, an LSTM is given by the following update rules,
    \begin{align}
        &\textbf{f}_t = \sigma(\wm^f \xv_t + \um^f \hv_{t-1} + b^f) \\
        &\textbf{i}_t = \sigma(\wm^i \xv_t + \um^i \hv_{t-1} + b^i) \\
        &\textbf{o}_t = \sigma(\wm^o \xv_t + \um^o \hv_{t-1} + b^o) \\
        &\tilde{\textbf{c}}_t = \tanh(\wm^{\tilde{c}} \xv_t + \um^{\tilde{c}} \hv_{t-1} + b^{\tilde{c}} ) \\
        &\textbf{c}_t = \textbf{f}_t \odot \textbf{c}_{t-1} + \textbf{i}_t \odot \tilde{\textbf{c}}_t \\
        &\hv_t = \textbf{o}_t \odot \tanh(\textbf{c}_t),
    \end{align}
where $\odot$ is element-wise multiplication. A more modern variation of the LSTM is that of the GRU, which employs only two gating mechanisms making it faster, and simpler to implement. GRUs have a similar motivation to LSTMS, but rather than using ``input'', ``output'', and ``forget'' gates, the GRU has a ``reset'' gate and an ``update'' gate. The ``update'' gate, $\textbf{z}_t$ determines the balance of old memory and new memory used in updating the new state, while the ``reset'' gate $\textbf{r}_t$ controls how much influence the input versus the previous state has on the new state. Mathematically, the GRU is defined as follows,
    \begin{align}
        &\textbf{z}_t = \sigma(\wm^z \xv_t + \um^z \hv_{t-1} + b^z) \\
        &\textbf{r}_t = \sigma(\wm^r \xv_t + \um^r \hv_{t-1} + b^r) \\
        &\textbf{h}_t = \textbf{z}_t \odot \hv_{t-1} + (1 - \textbf{z}_t) \odot \tanh \left(\wm^h \xv_t + \um^h ( \textbf{r}_t \odot \hv_{t-1}) + b^h \right).
    \end{align}
Despite being newer and quite similar in design, we will see that GRUs are not just less powerful than LSTMs, but only equally as powerful as SRNNs. 

In \cite{weiss}, Weiss et al. show how a LSTM can emulate an SKCM. To summarize their construction: the choice to increment ($+1$) or decrement ($-1$) is made in $\tilde{\textbf{c}}_t$ via the $\tanh$ function which naturally makes this decision as a function of the input token and previous state. A cell, $\textbf{c}_t$, can maintain the counter when $\textbf{i}_t = 1$ and $\textbf{f}_t = 1$. For comparison operations, the counter is completely visible in $h_t$. Extending this, Merrill \cite{mer} shows that, asymptotically, $\mathcal{L}(\textsf{IBFP-LSTM}) \subseteq \textbf{CL}$. Together these results give us the theorem below.
\begin{theorem} Let \textbf{SKCL} be the class of all simplified $k$-counter languages (SKCL), and \textbf{CL} be the class of all counter languages, then asymptotically,
   \[
    \textbf{SKCL} \subseteq \mathcal{L}(\textsf{IBFP-LSTM}) \subseteq \textbf{CL}.
   \]
\end{theorem}

\noindent On the other hand, Weiss et al. \cite{weiss} additionally discuss why a GRU is not capable of emulating a $k$ counter machine. Their argument, succinctly is that the update rule for $\hv_t$ forces the state values to be bounded between $-1$ and $1$. In the IBFP setting, it is not possible to perform unbounded counting in this representation. Of course, for the same reason, a SRNN is also not capable of emulating a $k$ counter machine. Merrill further shows that both GRUs and SRNNs are limited to asymptotically accepting regular languages \cite{mer}.

\begin{theorem} The following relationship holds asymptotically,
   \[
    \mathcal{L}(\textsf{IBFP-GRU}) = \mathcal{L}(\textsf{IBFP-SRNN}) = \textbf{REGULAR}.
   \]
\end{theorem}

\noindent Correspondingly, they show that the state complexity of a memory cell captures these separations.

\begin{theorem} Let $\hv_n^g$,  and $\hv_n^s$ be GRU and SRNN cells respectively. Then asymptotically,
\[
   \textarc{m}(\hv_n^g), \textarc{m}(\hv_n^s) \in O(1).
\]
\end{theorem}

\begin{theorem} Let $\textbf{c}_n \in \mathds{R}^k$ be an LSTM cell state. Then asymptotically,
\[
    \textarc{m}(\textbf{c}_n) \in O(n^k).
\]
\end{theorem}
\noindent Overall, we have seen evidence that LSTMs are intrinsically related to counting machines in terms of function and expressibility. Meanwhile, GRUs and SRNNs are closely related to regular languages giving them have less formal power. However both GRUs and SRNNs are far simpler to implement, which in some constrained settings could be beneficial.

\subsection{Transformer Networks}
Transformer Networks \cite{attn} are a new, and increasingly popular architecture designed with the goal of improving machine translation. One practical advantage of Transformers, is they are far more parallelizable than RNNs, and learn long term dependencies more efficiently. Introduced by Vaswani et al., in their paper ``Attention Is All You Need'' \cite{attn}, Transformers completely forgo classical recurrent connections for the notion of \textit{attention}. Attention is a mechanism which enables a network to selectively recall a specific encoder state based on observed information. This process is modeled as a database-esque retrieval process involving a query $\qv \in \mathds{R}^\ell$, a $n \times \ell$ matrix of key vectors $\km$, and a $n \times d$, matrix of value vectors $\vm$. However, unlike a traditional database retrieval, this process is ``soft'' meaning that for a given query we do not get the exact value back, but instead a sum of the values weighted by how similar each key is to the query. The compact notation expressing this idea is simply,
    \[
        \textsf{attention}(\qv, \km, \vm) = \textsf{softmax}(\qv \km^\top) \vm.
    \]
Transformers use multiple attention heads in parallel, while simultaneously linearly projecting the queries, keys, and values with different, learned transformations. This enables them to leverage different trends across different embeddings. Then at the end, the output from each transformation is concatenated. This defines the notion of \textit{multihead attention} \cite{attn}, 
    \begin{align}
        &\textbf{a}_i = \textsf{attention}(\wm^{q_i} \qv, \wm^{K_i}\km, \wm^{V_i}\vm) \\
        &\textsf{multihead}(\qv, \km, \vm) = \textbf{a}_1 || \cdots || \textbf{a}_n.
    \end{align}
Finally, the Transformer network studied in \cite{mer}, is a variant adapted for language processing tasks \cite{rad}, defined by,
    \begin{align*}
        &\qv_t, \textbf{k}_t, \vv_t = \wm^{q_t}\xv_t, \wm^{k_t}\xv_t, \wm^{v_t}\xv_t \\
        &\hv_t = \sigma(\wm^h \textsf{multihead}(\qv, \km, \vm)).
    \end{align*}
Surprisingly, despite many transformers taking the spotlight in many top-performing natural language processing architectures, they are very weak asymptotically. A glaring flaw of the transformer architecture, is their positional invariance. Vaswani et al. incorporate a trick to augment the network with positional encodings to fix this, however they use periodic functions, which repeat asymptotically \cite{attn}. As such, in this setting Transformer Networks have quite limited power.
    \begin{theorem} Asymptotically the following containment holds,
        \[
            \textbf{REGULAR} \not \subset \mathcal{L}(\textsf{IBFP-TN}). 
        \]
    \end{theorem}
\noindent Interestingly, Transformer Networks do have high state-complexity. In particular, it is known that  $\textarc{m}(\vm_n) \in 2^{\Theta(n)}$ \cite{mer}. This may explain their practical success, since of course in a completely finite sense, one can add positional encodings via the trick in \cite{attn}. Subsequently, this result may be a drastic underestimation of their potential towards empirical language modeling.

\subsection{Memory Augmented Neural Networks}
All of the previously described networks maintain bizarre implicit representations of the data, which always seem impede their formal capacity under the IBFP assumptions. One might wonder, what if we just used constructs from formal language theory explicitly as part of the model? This question motivates the idea of Memory Augmented Neural Networks (MANNs) which integrate explicit memory constructs such as stacks or tapes into neural network architectures. A drawback, is that implementing these constructs in a fully differentiable fashion, adds a considerable number of new hyperparameters to optimize and to analyse. For this reason, the literature including both formal and empirical focused work on the power of MANNs is relatively sparse compared to the other methods. However, as we will see these models have been used to perform interesting formal language focused tasks, and have very high representational power. 

Recently, there have been a few proposals on ways to augment networks with differentiable stacks, the most relevent of which are \cite{gref} \cite{joulin}, and \cite{suz}. Grefenstette et al.  empirically demonstrated that memory augmented LSTMs had superior performance over vanilla LSTMs on certain transduction tasks \cite{gref}. Meanwhile Joulin et al. \cite{joulin} studied the ability of MANNs to learn variations of languages such as $a^nb^n$, however (unsurprisingly given theorem 3) they did not observe any superior performance in the augmented network over LSTMs. Lastly, \cite{suz} studied the ability for MANNs to learn $\mathcal{D}_{>1}$ languages. Due to its relevance to our work, and its relative similarity to the stack model described in \cite{joulin}, we will introduce Suzgun et al.'s Stack-RNN, and remark on the expressiveness of a more abstract variant. 

The aptly named Stack-RNN \cite{suz} integrates a differentiable stack into a recurrent neural network by the following update rules.
    \begin{align}
        &\tilde{\hv}_{t-1} = \hv_{t-1} + \wm_{sh} \svec^{(0)}_{t-1} \\
        &\hv_t = \tanh(\wm^x \xv_t + \um^h \tilde{\hv}_{t-1} + b^h) \\
        &\yv_t  = \sigma(\wm^y \hv_t + b^y) \\
        &\av_t = \textsf{softmax}\left(\wm^a \hv_t\right) \\
        & \nv_t = \sigma(\wm^n \hv_t) \\
        & \svec^{(0)}_t = \av^{(0)}_t \nv_t + \av^{(1)}_t \sv^{(1)}_{t-1} \\
        & \svec^{(i)}_t = \av^{(0)}_t  \sv^{(i-1)}_{t-1} + \av^{(1)}_t \sv^{(i+1)}_{t-1}
    \end{align}
The stack is maintained as $\svec_t = \svec^{(0)}_t \svec^{(1)}_t \cdots \svec^{(k)}_t$. The vector $\av_t = [\av_t^{(0)} \av_t^{(1)}]$ encodes the \texttt{PUSH} and \texttt{POP} operations. For example, if $\av_t^{(0)} = 1$ (\texttt{PUSH}) then all the stack elements $\sv_t^{(i)}$ are pushed down (Eq. 23), and the topmost element is changed to the new value $\nv_t$. Likewise, if $\av_t^{(1)} = 1$ (\texttt{POP}), then each stack element is shifted down. Moreover, the hidden state (Eq. 17, 18) equations have been adapted to take into consideration the topmost stack element.

Unsurprisingly, Neural Stacks have been quite successful at certain formal language tasks, and similar abstract models which implement stack interfaces, as studied in \cite{mer}, are known to inherit a great deal of expressive power from the addition of a stack. Suzgun et al. show that Neural Stacks can learn $\mathcal{D}_1$, $\mathcal{D}_2$, $\mathcal{D}_3$, and even $\mathcal{D}_6$ languages with 99\%-100\% accuracy (often closer to 100\%) on both training and test sets \cite{suz}. Merrill \cite{mer}, shows that for his abstract, definition of a neural stack the following holds.

\begin{theorem} Let $\textbf{S}_n \in \mathds{R}^{nk}$ be a neural stack with a feed-forward controller. Then,
\[
    \textarc{m}\left(\textbf{S}_n \right) \in  2^{\Theta(n)}.
\]
\end{theorem}

\noindent Beyond stacks, one can also augment neural networks with some notion of a memory tape to make Neural Turing Machines. Originally proposed in \cite{graves}, Neural Turing Machines are notoriously complex to implement, and difficult to train. In fact, most follow up work to \cite{graves} is focused on improving the original NTM network design to make it more usable \cite{zaremba} \cite{kurach} \cite{yang} \cite{graves2016} \cite{gul}. 
The greater model flexibility also makes it harder to analyse asymptotically. However, it seems believable to us that in the IBFP setting, a result like 
    \[
        \mathcal{L}(\textsf{IBFP-NTM}) \subseteq \textbf{CONTEXT-SENSITIVE} 
    \]
could hold, and that NTMs would have exponential state complexity. Due to the complexity of a full NTM, we refrain from defining it formally here. However, Suzgun et al. \cite{suz} define what they call a Baby-NTM, which implements five operations on the ``tape'' represented by a real vector in a Neural Stack like fashion. For a tape $[a,b,c,d,e]$, these operations are,
    \begin{align*}
        \texttt{ROTATE-RIGHT} : [e, a, b, c, d] \\
        \texttt{ROTATE-LEFT} :  [b, c, d, e, a] \\
        \texttt{NO-OP} : [a, b, c, d, e] \\
        \texttt{POP-RIGHT} : [0, a, b, c, d] \\
        \texttt{POP-LEFT} : [b, c, d, e, 0].
    \end{align*}
We also omit the specific implementation details of the Baby-NTM, as it is quite similar to the neural stack. Also the core point we want to convey revolves around the fact that this architecture is more complicated, and more flexible than the Neural Stack. However, despite this \cite{suz} observed the Neural Stack had slightly better performance on learning $\mathcal{D}_1$, $\mathcal{D}_2$, $\mathcal{D}_3$, and $\mathcal{D}_6$ languages. This performance gap is unlikely representative of the model's asymptotic power, however it does highlight some of the emperical pitfalls of augmenting a model with more powerful constructs. 

Overall, MANNs certainly show a great deal of promise and potential to learn complex formal languages. However, due to their relative novelty they are far less understood, and less mature than other models we discussed in this paper. As such, there is certainly a great deal of room for future work leveraging these architectures.

\section{Summary} \label{sec-summary}
Understanding how well different neural network architectures decide membership for different classes of formal languages is a fundamental problem that is closely connected to our aspirations in the SafeDocs project. We surveyed a range of results, both theoretical and empirical quantifying the expressiveness of many state-of-the-art network architectures such as Convolutional Neural Networks, Recurrent Neural Networks, Transformer Networks, and Memory Augmented Neural Networks. We saw that networks such as CNNs, SRNNs, and GRUs are easier to implement and train, but formally may lack the capacity to learn complex languages. Then, we saw that the well-known Transformer Network is asymptotically weak, but posses high state complexity, which may explain its practical success. Finally, we explored augmenting networks with various notions of differentiable memory, and observed that this drastically increased their power at the cost of lower usability. Looking forward, we expect that our current work will discover similar results but concerning the generative capability of different networks, with respect to the Chomsky Hierarchy. 



\begin{thebibliography}{10}

\bibitem{gru}
Kyunghyun Cho, Bart van Merrienboer, {\c{C}}aglar G{\"{u}}l{\c{c}}ehre, Fethi
  Bougares, Holger Schwenk, and Yoshua Bengio.
\newblock Learning phrase representations using {RNN} encoder-decoder for
  statistical machine translation.
\newblock {\em CoRR}, abs/1406.1078, 2014.

\bibitem{das}
Sreerupa Das, C~Lee Giles, and Guo-Zheng Sun.
\newblock Learning context-free grammars: Capabilities and limitations of a
  recurrent neural network with an external stack memory.
\newblock In {\em Proceedings of The Fourteenth Annual Conference of Cognitive
  Science Society}, 1992.

\bibitem{elman}
Jeffrey~L. Elman.
\newblock Finding structure in time.
\newblock {\em Cognitive Science}, 14:179--211, 1990.

\bibitem{fischer1968counter}
Patrick~C Fischer, Albert~R Meyer, and Arnold~L Rosenberg.
\newblock Counter machines and counter languages.
\newblock {\em Mathematical systems theory}, 2(3):265--283, 1968.

\bibitem{ian}
Ian~J. Goodfellow, Jean Pouget-Abadie, Mehdi Mirza, Bing Xu, David
  Warde-Farley, Sherjil Ozair, Aaron~C. Courville, and Yoshua Bengio.
\newblock Generative adversarial networks.
\newblock {\em ArXiv}, abs/1406.2661, 2014.

\bibitem{graves}
Alex Graves, Greg Wayne, and Ivo Danihelka.
\newblock Neural {T}uring machines.
\newblock {\em ArXiv}, abs/1410.5401, 2014.

\bibitem{graves2016}
Alex Graves, Greg Wayne, Malcolm Reynolds, Tim Harley, Ivo Danihelka, Agnieszka
  Grabska-Barwinska, Sergio~G{\'o}mez Colmenarejo, Edward Grefenstette, Tiago
  Ramalho, John Agapiou, Adri{\`a}~Puigdom{\`e}nech Badia, Karl~Moritz Hermann,
  Yori Zwols, Georg Ostrovski, Adam Cain, Helen. King, C.~Summerfield, Phil
  Blunsom, Koray Kavukcuoglu, and Demis Hassabis.
\newblock Hybrid computing using a neural network with dynamic external memory.
\newblock {\em Nature}, 538:471--476, 2016.

\bibitem{gref}
Edward Grefenstette, Karl~Moritz Hermann, Mustafa Suleyman, and Phil Blunsom.
\newblock Learning to transduce with unbounded memory.
\newblock {\em CoRR}, abs/1506.02516, 2015.

\bibitem{gul}
Caglar Gulcehre, Sarath Chandar, Kyunghyun Cho, and Yoshua Bengio.
\newblock Dynamic neural {T}uring machine with continuous and discrete
  addressing schemes.
\newblock {\em Neural Comput.}, 30(4):857–884, April 2018.

\bibitem{lstm}
Sepp Hochreiter and J\"{u}rgen Schmidhuber.
\newblock Long short-term memory.
\newblock {\em Neural Comput.}, 9(8):1735–1780, Nov 1997.

\bibitem{holldobler}
Steffen H{\"o}lldobler, Yvonne Kalinke, and Helko Lehmann.
\newblock Designing a counter: Another case study of dynamics and activation
  landscapes in recurrent networks.
\newblock In {\em KI}, 1997.

\bibitem{jager2012formal}
Gerhard J{\"a}ger and James Rogers.
\newblock Formal language theory: refining the {C}homsky hierarchy.
\newblock {\em Philosophical Transactions of the Royal Society B: Biological
  Sciences}, 367(1598):1956--1970, 2012.

\bibitem{joulin}
Armand Joulin and Tomas Mikolov.
\newblock Inferring algorithmic patterns with stack-augmented recurrent nets.
\newblock In {\em NIPS}, 2015.

\bibitem{kurach}
Karol Kurach, Marcin Andrychowicz, and Ilya Sutskever.
\newblock Neural random access machines.
\newblock {\em ERCIM News}, 2016, 2015.

\bibitem{mer}
William Merrill.
\newblock Sequential neural networks as automata.
\newblock {\em ArXiv}, abs/1906.01615, 2019.

\bibitem{rad}
Alec Radford.
\newblock Improving language understanding by generative pre-training.
\newblock 2018.

\bibitem{rodriguez}
Paul Rodriguez and Janet Wiles.
\newblock Recurrent neural networks can learn to implement symbol-sensitive
  counting.
\newblock In M.~I. Jordan, M.~J. Kearns, and S.~A. Solla, editors, {\em
  Advances in Neural Information Processing Systems 10}, pages 87--93. MIT
  Press, 1998.

\bibitem{sieg}
Hava~T Siegelmann and Eduardo~D Sontag.
\newblock Analog computation via neural networks.
\newblock {\em Theoretical Computer Science}, 131(2):331 -- 360, 1994.

\bibitem{steijvers1996recurrent}
Mark Steijvers and Peter Gr{\"u}nwald.
\newblock A recurrent network that performs a context-sensitive prediction
  task.
\newblock In {\em Proceedings of the 18th Annual Conference of the Cognitive
  Science Society}, pages 335--339, 1996.

\bibitem{suz}
Mirac Suzgun, Sebastian Gehrmann, Yonatan Belinkov, and Stuart~M. Shieber.
\newblock Memory-augmented recurrent neural networks can learn generalized
  {D}yck languages.
\newblock {\em ArXiv}, abs/1911.03329, 2019.

\bibitem{tonkes}
Bradley Tonkes and Janet Wiles.
\newblock Learning a context-free task with a recurrent neural network: An
  analysis of stability.
\newblock In {\em In Proceedings of the Fourth Biennial Conference of the
  Australasian Cognitive Science Society}. Citeseer, 1997.

\bibitem{attn}
Ashish Vaswani, Noam Shazeer, Niki Parmar, Jakob Uszkoreit, Llion Jones,
  Aidan~N. Gomez, Lukasz Kaiser, and Illia Polosukhin.
\newblock Attention is all you need.
\newblock In {\em NIPS}, 2017.

\bibitem{weiss}
Gail Weiss, Yoav Goldberg, and Eran Yahav.
\newblock On the practical computational power of finite precision rnns for
  language recognition.
\newblock {\em ArXiv}, abs/1805.04908, 2018.

\bibitem{yang}
Greg Yang.
\newblock Lie access neural {T}uring machine.
\newblock {\em ArXiv}, abs/1602.08671, 2016.

\bibitem{yin}
Wenpeng Yin, Katharina Kann, Mo~Yu, and Hinrich Sch{\"{u}}tze.
\newblock Comparative study of {CNN} and {RNN} for natural language processing.
\newblock {\em CoRR}, abs/1702.01923, 2017.

\bibitem{zaremba}
Wojciech Zaremba and Ilya Sutskever.
\newblock Reinforcement learning neural {T}uring machines.
\newblock {\em CoRR}, abs/1505.00521, 2015.

\end{thebibliography}

\end{document}